%% file: 0-main.tex
\documentclass{article}

\usepackage[nonatbib,preprint]{neurips_2020}

\usepackage[utf8]{inputenc} 
\usepackage[T1]{fontenc}    
\usepackage{hyperref}       
\usepackage{url}            
\usepackage{booktabs}       
\usepackage{amsfonts}       
\usepackage{nicefrac}       
\usepackage{microtype}      
\usepackage[numbers]{natbib}

\usepackage{graphicx}
\graphicspath{{Figures}}
\DeclareGraphicsExtensions{.pdf,.png}
\usepackage{subcaption}
\usepackage{bbm}
\usepackage{amssymb}
\usepackage{amsthm}
\usepackage{booktabs} 
\usepackage{multicol}
\usepackage{soul}
\usepackage{wrapfig}

\usepackage[dvipsnames]{xcolor}
\input{math_commands.tex}
\usepackage{mathtools}

\usepackage{algorithm}
\usepackage[noend]{algpseudocode}
\usepackage{mathtools}
\usepackage{thmtools}
\usepackage{thm-restate}
\usepackage{caption}

\newtheorem{theorem}{Theorem}

\newtheorem{corollary}{Corollary}[theorem]
\newtheorem{lemma}{Lemma}
\newtheorem{assumption}{Assumption}

\newcommand\numberthis{\addtocounter{equation}{1}\tag{\theequation}}

\makeatletter
\algrenewcommand\ALG@beginalgorithmic{\footnotesize}
\makeatother

\algnewcommand{\LeftComment}[1][\footnotesize]{\State #1}

\usepackage{titlesec}
\titlespacing\section{0pt}{0pt plus 2pt minus 1pt}{0pt plus 2pt minus 1pt}
\titlespacing\subsection{0pt}{0pt plus 2pt minus 1pt}{0pt plus 2pt minus 1pt}
\titlespacing\subsubsection{0pt}{0pt plus 2pt minus 1pt}{0pt plus 2pt minus 1pt}

\algnewcommand{\BlockComment}[1][\footnotesize]{\State \textcolor{blue}{/*{#1}*/}}
\algrenewcommand\algorithmiccomment[1]{\hfill \textcolor{blue}{#1}}

\title{Learning to be Safe: Deep RL with a Safety Critic}

\author{ %
  Krishnan Srinivasan$^1$, Benjamin Eysenbach$^2$, Sehoon Ha$^3$, Jie Tan$^4$, Chelsea Finn$^{1,4}$ \\
  $^1$Stanford University, $^2$Carnegie Mellon University, $^3$Georgia Tech, $^4$Robotics at Google \\
  \texttt{krshna@stanford.edu} \\
}

\begin{document}

\maketitle

\begin{abstract}
Safety is an essential component for deploying reinforcement learning (RL) algorithms in real-world scenarios, and is critical during the learning process itself.
A natural first approach toward safe RL is to manually specify constraints on the policy's behavior. 
However, just as learning has enabled progress in large-scale development of AI systems, learning safety specifications may also be necessary to ensure safety in messy open-world environments where manual safety specifications cannot scale.
Akin to how humans learn incrementally starting in child-safe environments, we propose to \emph{learn} how to be safe in one set of tasks and environments, and then use that learned intuition to constrain future behaviors when learning new, modified tasks. We empirically study this form of \emph{safety-constrained transfer learning} in three challenging domains: simulated navigation, quadruped locomotion, and dexterous in-hand manipulation. In comparison to standard deep RL techniques and prior approaches to safe RL, we find that our method enables the learning of new tasks and in new environments with both substantially fewer safety incidents, such as falling or dropping an object, and faster, more stable learning. This suggests a path forward not only for safer RL systems, but also for more effective RL systems.
\end{abstract}

\section{Introduction}
\label{sec:intro}
\input{1-introduction.tex}
\section{Related Work}
\label{sec:related}
\input{2-related.tex}
\section{Preliminaries}
\label{sec:prelim}
\input{3-preliminaries.tex}

\section{Problem Statement}
\label{sec:problem}
\input{4-problem.tex}

\section{Safety Q-functions for RL}
\label{sec:method}

\input{5-method.tex}

\section{Experiments}
\label{sec:experiments}
\input{6-experiments.tex}
\section{Discussion}
\label{sec:conclusion}
\input{7-conclusion.tex}


\bibliographystyle{plainnat}
\bibliography{references.bib}

\appendix
\input{appendix.tex}

\end{document}

%% file: math_commands.tex
\usepackage{amsmath,amsfonts,bm}

\def\eqref#1{equation~\ref{#1}}

\def\1{\bm{1}}

\def\eps{{\epsilon}}

\DeclareMathAlphabet{\mathsfit}{\encodingdefault}{\sfdefault}{m}{sl}
\SetMathAlphabet{\mathsfit}{bold}{\encodingdefault}{\sfdefault}{bx}{n}

\def\gA{{\mathcal{A}}}

\def\gD{{\mathcal{D}}}

\def\gH{{\mathcal{H}}}
\def\gI{{\mathcal{I}}}

\def\gS{{\mathcal{S}}}
\def\gT{{\mathcal{T}}}

\newcommand{\E}{\mathbb{E}}

\newcommand{\qsafe}{\ensuremath{Q_{\text{safe}}}}
\newcommand{\esafe}{\ensuremath{\epsilon_\text{safe}}}
\newcommand{\gsafe}{\ensuremath{\gamma_\text{safe}}}

\newcommand{\sunsafe}{\ensuremath{\mathcal{S}_{\text{unsafe}}}}
\newcommand{\ptask}{\ensuremath{\gT_\text{pre}}}
\newcommand{\gtask}{\ensuremath{\gT_{\text{target}}}}

%% file: 1-introduction.tex
While reinforcement learning systems have demonstrated impressive potential in a variety of domains, including video games~\cite{Mnih2013PlayingAW} and robotic control in laboratory environments~\cite{levine2015end}, safety remains a critical bottleneck when deploying these systems for real world problems. One natural approach to ensure safety is to manually impose constraints onto the learning process, to prevent the agent from taking actions or entering states that are too risky. For example, this could be done by manually shaping a reward function or imposing constraints on the policy to avoid actions that may lead to unsafe outcomes~\cite{45450}. 
However, the recent successes of machine learning techniques in a wide range of applications indicate advantages to avoiding manual specification and engineering, as such manual approaches will not generalize to new environments or robots~\cite{robonet}. This begs the question: why should we make an exception with safety? In this paper, we propose to explore the following question: can we \emph{learn} how to be safe, avoiding the need to manually specify states and actions that are risky?

\begin{figure*}[t]
    \centering
    \includegraphics[width=0.99\linewidth]{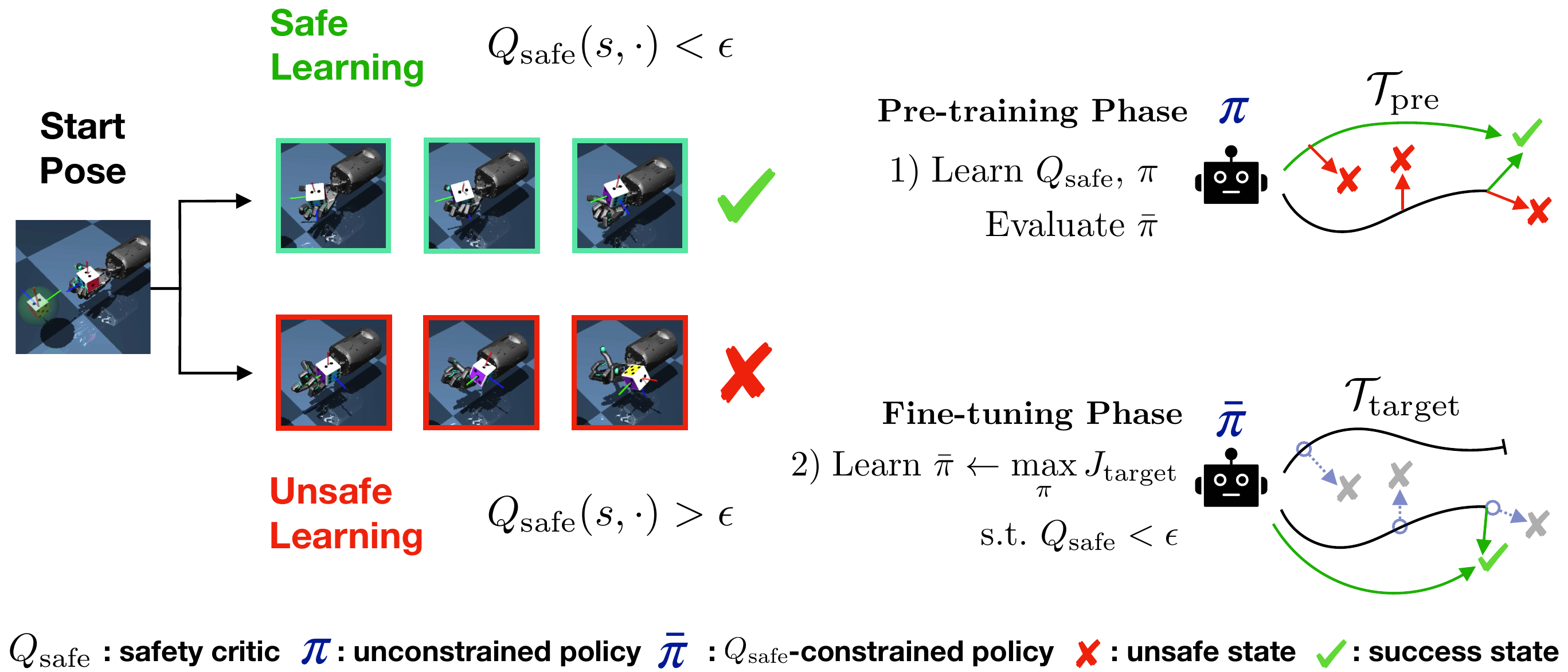}
    \vspace{-0.1cm}
    \caption{Our approach enables safe RL by pre-training a critic $Q_\text{safe}^{\bar{\pi}}$ that is trained by constraining the actions of a policy $\pi$. This yields safe policy iterates (resulting from optimizing Eq.~\ref{eq:pol-ft}) that achieve safe episodes throughout fine-tuning, while standard RL approaches will visit unsafe states during adaptation, which results in more failed episodes.}
    \label{fig:hand-safe}
    \vspace{-1em}
\end{figure*}

At first, learning and safety may seem inherently at odds, since learning to be safe requires visiting unsafe states. However, when children learn to walk they can leverage the positive and negative experience of getting up and falling down to gradually begin to stand, walk, and eventually run. Throughout this process, they learn how to move around in the world safely and adapt these skills as their bodies grow and are exposed to new and unseen environments, all while limiting how badly and often they are injured. This way of safely exploring while learning is something that can not only be used when learning to walk, but in other areas where safety is critical not only for survival, but generally helpful for successfully learning the task.
For instance, an agent can more safely learn to drive a car if it already knows how to avoid collisions.



Motivated by how humans learn to acquire skills that become gradually harder and more risky, we believe that our reinforcement learning agents can similarly leverage previous experience to understand general safety precautions in safety-training environments, and use this understanding to avoid failures when learning \emph{new behaviors}. This previous experience could come from non-safety-critical environments such as when a human is present, when a robot is moving slowly, in simulation, or from pre-collected offline data of safety incidents such as car accidents.  To this end, we propose to learn task-agnostic models of safety, and then use these models such as preventing unsafe behavior when learning new tasks via reinforcement. By doing so, we can ensure that the latter reinforcement learning process itself is safe, which in principle, can enable real-world deployment. 
The primary contribution of this work is a framework for safe reinforcement learning, by learning safety precautions from previous experience. Our approach, safety Q-functions for reinforcement learning (SQRL), learns a critic that evaluates whether a state, action pair will lead to unsafe behavior, under a policy that is constrained by the safety-critic itself. This is achieved by concurrently training the safety-critic and policy in separate pre-training and task fine-tuning phases, visualized in Figure~\ref{fig:hand-safe}. During pre-training, the agent is allowed to explore and learn about unsafe behaviors. In the fine-tuning phase, we train the policy on a new target task, while simultaneously constraining the policy's updates and selected actions with the safety-critic. We evaluate our method in three challenging continuous control environments: 2D navigation, quadrupedal robot locomotion, and dexterous manipulation with a five-fingered hand. In comparison to prior state-of-the-art RL methods and approaches to safety, we find that SQRL provides consistent and substantial gains in terms of both the safety of learning and the learning efficiency on new tasks.

%% file: 2-related.tex
Our method builds on a rich body of work on safe RL~\cite{garcia2015comprehensive}. Prior works define safety in a variety of ways, using constraints on the expected return or cumulative costs~\cite{heger1994consideration, achiam2017constrained}, using risk measures such as Conditional Value at Risk and percentile estimates~\cite{Tamar2014PolicyGB,chow2014algorithms, duan2020distributional, ma2020distributional}, and defining regions of the state space that result in catastrophic or near-catastrophic failures \cite{lipton2016combating, eysenbach2017leave, fisac2019bridging}. We define safety as constraining the probability that a catastrophic failure will occur below a specified threshold, similar to how it is defined in prior work on MDPs with constrained probability of constraint violation (CPMDPs)~\cite{geibel2006reinforcement}.

Researchers have determined a wide variety of challenges underlying safe RL, including lower-bounds on off-policy learning and evaluation~\cite{thomas2015high,dann2018policy}, robustness to perturbations~\cite{smirnova2019distributionally,pinto2017robust,xu2010distributionally}, learning to reset to safe states~\cite{moldovan2012safe, eysenbach2017leave}, and verifying the safety of a final policy~\cite{julian2019verifying}. One common focus of prior works is to learn a safe policy by minimizing worst-case discounted cost using conditional value at risk (CVaR) formulation~\cite{heger1994consideration,chow2017risk,tang2019worst}. However, these methods do not guarantee or offer ways to reason about the safety of policies attempted during the learning process. In contrast, we present a method of learning a safety layer that can adapt in parallel with a fine-tuned policy in an online fashion, making it better suited in the transfer learning setting where the environment or task has changed and the trained agent must safely adapt.

While there are many approaches to safety in RL, constrained MDPs are commonly used~\cite{altman1999constrained, achiam2017constrained, Tessler2018RewardCP}, where the agent optimizes the reward signal given constraints on the accumulated cost during the rollout. In work by \cite{chow2018lyapunov, Chow2019LyapunovbasedSP}, Lyapunov state-action functions are constructed using a feasible baseline policy, to enable safe Q-learning, policy gradient, and actor gradient methods, which map a learned policy into a space of policies that satisfy the Lyapunov safety condition. Similar to our work, this approach instills safety into the training process by directly evaluating every state's constraint cost. However, unlike these works, we do not assume well-defined, actionable safety criterion, and instead only access a sparse safety label which the agent uses to infer an actionable safety criterion and transfer it to new task instances.


There have been many approaches to handling risk and safety in RL that have varied at an architectural level, which bear some similarities to our work. {In \cite{geibel2005risk}, Q-learning is adapted with an additional risk value estimate based on unsafe states to learn a deterministic policy that is risk-averse. } \cite{hans2008safe} proposed controlling a plant with RL, while learning a ``safety function'' from data to estimate a state's degree of safety, as well as a ``backup policy'' for returning the plant from a critical state to a safe regime. {However, neither approach tries transferring learned safety to a different task/environment with modified rewards or dynamics.}  \cite{eysenbach2017leave} showed a similar approach in a robotics domain, learning a ``reset policy'' to return a robot to its initial state distribution, which triggered if a corresponding value function drops below a certain threshold.
Our model builds on this idea by directly learning to predict failures, using this predictor to \emph{preemptively} filter unsafe actions, removing the need for an additional reset policy.
Additionally, we use our learned constraints to provide constrained updates to our policy during fine-tuning.

Finally, a number of works have learned models that predict the probability of future collisions for safe planning and control of mobile robots~\cite{richter2017safe,kahn2017uncertainty}. We instead aim to develop a general framework for many different modes of safety, and empirically observe the importance of several new design choices introduced in our algorithm on challenging continuous control problems.

%% file: 3-preliminaries.tex


For our method, we consider the standard RL problem on a Markov Decision Process (MDP), which is specified by the tuple $\langle\mathcal{S}, \mathcal{A}, \gamma, r, P, \mu \rangle$. $\mathcal{S}$ and $\mathcal{A}$ are the state and action spaces respectively, $\gamma$ is a discounting factor, $r: \mathcal{S}\times \mathcal{A} \to \mathbb{R}$ is the reward function, and $P(\cdot|s,a)$ and $\mu$ are the state transition dynamics and initial state distributions respectively.
The aim of entropy-regularized reinforcement learning~\cite{ziebart2010modeling} is to learn a policy $\pi_\theta: \gS \times \gA \rightarrow [0, 1]$ that maximizes the following objective:
\begin{equation}
    J(\theta) = \sum_{t=0}^{T-1} \mathbb{E}_{(\mathbf{s}_t, \mathbf{a}_t) \sim \rho_\pi}[r(\mathbf{s}_t, \mathbf{a}_t) + \alpha \mathcal{H}(\pi_\theta(\cdot | \mathbf{s}_t))].
\end{equation}
We use $\gH(\pi_\theta(\cdot|s_t))$ to denote the policy's action entropy and $\alpha > 0$ as a tuning parameter.

%% file: 4-problem.tex


In many robot learning settings, the conditions that lead to a catastrophic failure are difficult to fully specify before learning. While unsafe states are easily identifiable, such as the robot falling down or dropping an object, the task of formally specifying safety constraints that capture these states is non-trivial, and can be potentially biased or significantly hinder learning.

This is especially true if determining the agent's final behavior is difficult due to uncertainty about the underlying reward function and safety constraints. However, for many problems, the conditions for a catastrophic failure during an episode are easier to determine (e.g. determining that the robot has fallen down, or the object was dropped). Therefore, the challenge is to learn an optimal policy for a task while minimizing the frequency of catastrophic failures during training. To achieve this, we include 1) a safe pre-training environment, in which failures can be tolerated, and 2) a safety-incident indicator, $\mathcal{I}(s)$, which indicates if a given state is unsafe or not. In both the pre-training and target task environments, we treat these failure states as terminal. This setup enables the agent to learn to be safe and then transfer to new tasks without accumulating additional failure costs.

Under this framework, we propose learning in two phases: 1) learning an exploratory policy that solves a simpler/safer task in the pre-training environment, and 
2) \textcolor{black}{transferring the learned policy to} a more safety-critical target task with guarantees to safety. Our \emph{safety-aware} MDP for the pre-training and target tasks are respectively defined as $\mathcal{T_\text{pre}} = \langle \mathcal{S}, \mathcal{A}, P_\text{pre}, r_\text{pre}, \mu_\text{pre}, \mathcal{I} \rangle$, and $\mathcal{T_\text{target}} = \langle \mathcal{S}, \mathcal{A}, P_\text{target}, r_\text{target},\mu_\text{target},  \mathcal{I} \rangle$. 
After pre-training in a safety test-bed task ($\ptask$), the agent must optimize its expected return in the target training task ($\gtask$) while minimizing visits to unsafe states:
\begin{align*}
\max_{\pi} \sum_{t=0}^T \E_{(s_t,a_t)\sim \rho_\pi}
\left[r_\text{target}(s_t, a_t) \right] ~ \textrm{s.t.}\; \E_{s_t \sim{\rho_\pi}}\left[\mathcal{I}(s_t)\right] < \esafe. \;
\end{align*}
This objective implies that for the target task, the policy should always stay outside of the unsafe region $\sunsafe = \{s \;|\; \gI(s) = 1\}$ with probability $\esafe$.
This also introduces the notion of a target safety threshold, $\esafe$, which serves as an upper-bound on the expected risk \textcolor{black}{of a given policy}. When used as a single-step constraint over actions sampled during policy rollouts, it can guarantee that the policy $\pi$ is safe up to that threshold probability under certain assumptions. Importantly, we aim to impose this safety constraint not just at convergence, but throughout the training process. 

%% file: 5-method.tex
To address the problem formulation described above, we introduce the Safety Q-functions for Reinforcement Learning (SQRL), which simultaneously learns a policy and notion of safety in the first phase, and later fine-tunes the policy to the target task using the learned safety precautions. In turn, the safety in the second phase of training can be ensured during the learning process itself. This is done by the safety-critic, $\qsafe^{\bar{\pi}}$, which estimates future failure probability of a safety-constrained policy given a state, action pair. The safety-constraints learned by $\qsafe^{\bar{\pi}}$ can induce the policy to be safe, even when the task has changed. In this section, we will describe our approach and analyze its performance.

\subsection{Pre-Training Phase}

During pre-training, the goal is to learn a safety-critic $\qsafe^{\bar{\pi}}$ and the optimal pre-trained policy $\pi_\text{pre}^*$ which will serve as the initialization for training on the target task.  For our problem, the optimal safety-critic for the pre-training task $\ptask$ estimates the following expectation for a given policy $\pi$:
\begin{equation}
Q_\text{safe}^{\pi}(s_t, a_t) = \gI(s_t) + (1-\gI(s_t))
    \sum_{t'=t+1}^T
    \E_{\substack{ s_{t'} \sim P_\text{pre}(\cdot | s_t, a_t) \text{ for } t'=t+1\\
    s_{t'} \sim P_\text{pre}, \pi \text{ for } t'>t+1
    }} \left[
    \gsafe^{t'-t} \mathcal{I} (s_{t'})
    \right],
\label{eq:failure-prob}
\end{equation}
The target failure estimate $Q_\text{safe}^\pi$ estimates the true probability that {$\pi$} will fail in the future if, starting at state $s_t$, it takes the action $a_t$. This raises two key questions: 1) how should such a model of safety be trained, 2) with what data do we train it, and 3) for what policy $\pi$ do we train it?


Regarding the first question, if the safety-critic were a binary classifier, supervised learning could be used to just estimate the safety labels for a single timestep. However, to learn the cumulative failure probability in the future, the safety-critic must reason over future timesteps also. Hence, the safety-critic is trained using dynamic programming, as in standard Q-learning, using a discount term, $\gsafe$, to limit how far in the past the failure signal is propagated. This cumulative discounted probability of failure is estimated by the Bellman equation:
\begin{equation*}
\hat{Q}_\text{safe}^{\pi} (s, a) = \gI(s) + (1 - \gI(s))
    \mathbb{E}_{\substack{s' \sim P_\text{pre}(\cdot|s,a)\\ a'\sim \pi(\cdot|s')}}
        \left[ \gsafe \hat{Q}_\text{safe}^{\pi}(s', a')\right].
\end{equation*}
Parametrizing the safety Q-function $\hat{Q}_\text{safe}^\pi$ as a neural network with parameters $\psi$ yields the following objective:
\begin{align}
    J_{\text{safe}}(\psi) = \E_{(s ,a, s', a')\sim \rho_\pi} \bigg[\bigg( \hat{Q}_\text{safe}^\pi(s,a) - \nonumber 
    \left(\gI(s) + (1 - \gI(s)) \gsafe \bar{Q}_\text{safe}^\pi(s', a')\right) \bigg)^2 \bigg],
\end{align}
where $\bar{Q}_\text{safe}$ corresponds to the delayed target network.

\begin{figure*}[ttt!]

\begin{minipage}[t]{2.76in}
\begin{algorithm}[H]
\caption{SQRL Pre-training}\label{alg:sqrl-pre}
\begin{algorithmic}[1]
\Procedure{Pretrain}{$n_\text{pre}, \ptask, \esafe, \gsafe$} \label{safesac-pretrain}
\State Initialize replay buffers, $\gD_\text{safe}, \gD_\text{offline}$
\State Initialize networks, $\hat{Q}_\text{safe}^\psi, \pi_\theta, Q_{\phi_1}, Q_{\phi_2}$.
\State $s \sim \mu_\text{pre}(\cdot)$
\For{$n_\text{pre}$ steps} 
\For{$n_\text{off}$ steps} 
    \State Sample action, $a \sim \pi_\theta(\cdot | s)$
    \State $s' \sim P_\text{pre}(\cdot|s,a)$
    \State $\gD_\text{offline}\text{.add}((s, s', a, r_\text{pre}(s, a)))$
    \State Apply SAC update to $\theta$, $\phi_1$, $\phi_2$, $\alpha$
    \If {$\gI(s') = 1$} $s' \sim \mu_\text{pre}(\cdot)$\EndIf
    \State $s \gets s'$
\EndFor
\LeftComment{{\fontfamily{cmtt}\selectfont \textcolor{blue}{// Collect $k$ on-policy rollouts}}}
\For{$i \gets 1,k$ episodes}
    \State $\tau_i \gets$ \Call{Rollout}{$\bar{\pi}_\theta, \ptask$}
    \State $\gD_\text{safe}\text{.add}(\tau_i)$ 
\EndFor
\State $\psi \gets \psi - \hat{\nabla}_\psi J_\text{safe}(\psi)$
\EndFor
\Return{$\pi_\theta, \hat{Q}_\text{safe}^\psi$} \label{safesac-pretrain-end}
\EndProcedure
\end{algorithmic}
\end{algorithm}
\end{minipage}
\hfill
\begin{minipage}[t]{2.76in}
\begin{algorithm}[H]
\begin{algorithmic}[1]
\caption{SQRL Fine-tuning}\label{alg:sqrl-fine}
\Procedure{Finetune}{$n_\text{target}, \gtask$, $\hat{Q}_\text{safe}$} \label{safesac-finetune}
\State  $s \sim \mu_\text{target}(\cdot)$, $\;\gD_\text{offline} \gets \{\}$
\For{$n_\text{target}$ steps}
    \State Sample action, $a \sim \bar{\pi}_\theta(\cdot |s)$
    \LeftComment{{\fontfamily{cmtt}\selectfont \textcolor{blue}{// where $\bar{\pi}_\theta~\gets~\Gamma(\pi_\theta)$ (Eq.~\ref{eq:safe-pi})}}}
    \State $s'~\sim~P_\text{target}(\cdot|s,a)$
    \State $\gD_\text{offline}\text{.add}((s, s', a, r_\text{target}(s, a)))$
    \LeftComment{{\fontfamily{cmtt}\selectfont \textcolor{blue}{// Dual gradient ascent on Eq.~\ref{eq:pol-ft}}}}
    \State{ $\theta~\gets~\theta~+~\lambda \hat{\nabla}_\theta J_{\text{target}}(\theta,\alpha,\nu)$}
    \State $\alpha~\gets~\alpha~-~\lambda \hat{\nabla}_\alpha J_\text{target}(\theta,\alpha,\nu)$
    \State $\nu~\gets~\nu~-~\lambda \hat{\nabla}_\nu J_\text{target}(\theta,\alpha,\nu)$
    
    \If{$\gI(s') = 1$} $s' \sim \mu_\text{target}(\cdot)$\EndIf
    \State $s \gets s'$
\EndFor
\Return{$\pi_\theta$} \label{safesac-finetune-end}
\EndProcedure
\end{algorithmic}
\end{algorithm}
\end{minipage}
\vspace{-5mm}
\end{figure*}

For the safety model to learn new tasks without failure, it must first explore a diverse set of state-action pairs, including unsafe state-action pairs, during pre-training. This motivates using maximum entropy reinforcement learning, which encourages exploration by maximizing both the reward and the entropy of the policy. In our implementation, we choose the soft actor-critic (SAC) algorithm~\cite{haarnoja2018soft}, though in principle, any off-policy algorithm can be used as long as it explores sufficiently. 

Finally, we must determine which policy to optimize the safety-critic under. If we use the entirety of training experience from SAC, which contains data from a mixture of all policies encountered, this may result in a safety Q-function that is too pessimistic. For example, if the mixture of policies includes a random, unsafe policy from the start of training, then even a cautious action in one state may be considered unsafe, because risky actions are observed afterwards. Therefore, we are faced with a dilemma: while we need diverse data that includes a range of unsafe conditions, we also need the actions taken to be reflective of those taken by safer downstream policies constrained by the safety-critic to avoid learning a pessimistic safety-critic. To address this conundrum, the safety-critic is optimized under the mixture of policies that are constrained by the safety-critic itself which we denote as $\bar{\pi}_i$ where $i$ is the iteration. This mitigates pessimism, since for a safe state-action pair in the data, subsequent behavior from that point will be constrained as safe, hence producing a reliable target label. 



To this end, the safety-critic and a stochastic safety-constrained policy $\bar{\pi}_\epsilon(a \mid s)$ are jointly optimized during pre-training.
Let the safety policy $\bar{\pi}_\epsilon$ be a policy that has zero probability of sampliing an action $a$ where $\hat{Q}_\text{safe}(s,a) \geq \epsilon$. Let the set of all policies where this holds be $\Pi_\text{safe}^\epsilon$. Next, the projection $\Gamma^\epsilon_\text{safe} : \pi \rightarrow \bar{\pi}_\epsilon \in \Pi^\epsilon_\text{safe}$ be the projection mapping any policy onto its closest policy in $\Pi_\text{safe}$. Consi a natural projection operator, using our safety-critic definition, that does the following:
\vspace{-.1cm}
\begin{equation}
    \bar{\pi}(a \mid s) \propto \begin{cases} \pi(a \mid s) & \text{if } \hat{Q}_\text{safe}^\pi(s, a) \le \epsilon \\ 0 & \text{otherwise} \end{cases}.
    \label{eq:safe-pi}
\end{equation}
By masking the output distribution over actions from $\pi$ to only sampling actions where this safety condition is met, the policy $\bar{\pi}$ is ensured to be safe.

Our pre-training process, summarized in Algorithm~\ref{alg:sqrl-pre}, proceeds as follows.
Let $\epsilon=\esafe$ define our target safety threshold, beyond which actions should be rejected for being too risky. Then, at each iteration of training, we collect data from our current policy with actions constrained by the current safety-critic (i.e. projecting our policy into $\Pi_\text{safe}^{\epsilon}$), add the data to our replay buffer, update our safety-critic under the mixture of policies represented in the replay buffer, and update our policy using a MaxEnt RL algorithm.
Pre-training returns a safe policy that solves the pre-training task, $\bar{\pi}_\text{pre}$, and a safety Q-function $\hat{Q}_\text{safe}$. 


\vspace{-0.1em}
\subsection{Fine-Tuning Phase}
\vspace{-0.1em}
In the fine-tuning phase, SQRL initializes the policy to the safety-constrained pre-training policy $\bar{\pi}_\text{pre}$ and fine-tunes to a new safety-critical target task $\mathcal{T_\text{target}}$. To do this, we formulate a safety-constrained MDP using the pre-trained safety-critic $\hat{Q}_\text{safe}$. While fine-tuning on the target task, all data is collected using policies that are constrained by $\hat{Q}_\text{safe}$, following the data collection approach used in pre-training, and the policy is updated with respect to the target task reward function. We additionally add a safety constraint cost to the policy objective to encourage the unconstrained policy $\pi'$ to sample actions that will fall within the distribution of $\Gamma^\epsilon_\text{safe}(\pi') = \bar{\pi}'_\epsilon$, while optimizing for expected return. This modifies the standard soft-actor policy objective to 
\begin{align*}
    J_\text{target}(\theta, \alpha, \nu) = \E_{\substack{(s, a) \sim \rho_{\bar{\pi}_\theta'}\\ a' \sim \pi'_\theta(\cdot | s)}} \Big[ 
    r_\text{target}(s, a) - \alpha(\log\pi_\theta'(\cdot | s) - \bar{\mathcal{H}}) + \nu\big(\esafe - \hat{Q}_\text{safe}(s, a')\big)\Big], \numberthis
    \label{eq:pol-ft}
\end{align*}
where $\alpha$ and $\nu$ are Lagrange multipliers for the entropy and safety constraints respectively, and where $\bar{\mathcal{H}}$ denotes the target entropy (a hyperparameter of SAC). Algorithm~\ref{alg:sqrl-fine} summarizes the fine-tuning process.

\vspace{-0.1em}
\subsection{Analysis}
\vspace{-0.1em}

Next, we theoretically analyze the safety of SQRL, specifically considering the safety of the learning process itself. 
For this analysis (but \emph{not} in our experiments), we make the following assumptions:
\begin{assumption}
The safety-critic $\hat{Q}_\text{safe}^{\bar{\pi}}$ is optimal such that, after pre-training, can estimate the true expected future failure probability given by ${Q}_\text{safe}^{\bar{\pi}}$ in Equation~\ref{eq:failure-prob}, for any experienced state-action pair.
\label{assump1-s5}
\end{assumption}
\begin{assumption}
The transition dynamics leading into failure states have a transition probability at least $\epsilon$. That is, for all unsafe states $s' \in \gS_\text{unsafe}$, $P(s'|s, a) > \epsilon$ or $P(s'|s, a)=0$  $\forall (s, a)$.
\label{assump2-s5}
\end{assumption}
\begin{assumption}
The support of the pre-training data for $\qsafe^\pi$ covers the states and actions observed during fine-tuning.
\label{assump3-s5}
\end{assumption}
\begin{assumption}
There always exists a ``safe'' action $a$ at every safe state $s \not\in \gS_\text{unsafe}$ that leads to another safe state (i.e. $\exists a, s' : \gI(s') = 0 \; \text{s.t.} \; P(s' | s,a) > 0$ if $\gI(s) = 0$).
\label{assump4-s5}
\end{assumption}

\begin{lemma}
For any policy $\bar{\pi} \in \Pi_\text{safe}^\epsilon$, the discounted probability of it failing in the future, given by $\E_{s_{t'} \sim \rho_{\bar{\pi}}|s_t}(\gsafe^t \gI(s_{t'}) \mid t' > t)$, is less than or equal to $\esafe$, given a safety-critic $\qsafe^{\bar{\pi}}$.
\label{lem:1}
\end{lemma}

Assuming that $\hat{Q}_\text{safe}^{\bar{\pi}} = Q_\text{safe}^{\bar{\pi}}$ (Assumption 1), for the base case we can assume for a policy $\bar{\pi}$ starting in state $s_0 \in \mathcal{S}_\text{safe},$ there is always an action $a_0$ for which $\pi(a_0|s_0) > 0$ and $\hat{Q}_\text{safe}(s_0,a_0) < \esafe$:
\begin{equation*}
    \hat{Q}_\text{safe}^\pi(s_0, a_0) = \E_{\bar{\pi}}
    \left[\gsafe^t \mathcal{I}(s_t) | t \geq 0\right] = \Pr(\mathcal{I}(s_0) = 1) + \E_{\pi}\left[\gsafe^t \mathcal{I}(s_t) | t > 0\right] < \esafe.
\end{equation*}
From recursion, it follows that when the step probability of failing is below $\esafe,$ if the policy $\bar{\pi}$ masks the actions sampled by $\pi$ to those that are within the safety threshold, $\hat{Q}_\text{safe}^{\bar{\pi}}(s_t, a_t) < \esafe$, $a_t \sim \bar{\pi}_\epsilon(s_t)$ has failure probability less than $\hat{Q}_\text{safe}^{\bar{\pi}}(s, a).$  This is because there are actions in the support of the expectation of $\hat{Q}_\text{safe}^{\bar{\pi}}(s, a)$ where $\Pr(\mathcal{I}(s_t) = 1) < \esafe,$ or else that state's cumulative failure probability would be $> \epsilon.$ From Assumptions 2 and 3, this means that these actions can be sampled by $\bar{\pi}$, and from the definition of the masking operator $\Gamma_\text{safe}^\epsilon$, actions that would violate the constraint are masked out. A visual representation of this concept is presented in Fig.~\ref{fig:safety_mask} in Appendix~\ref{appendix:A}.
\begin{theorem}
Optimizing the policy learning objective in Eq.~\ref{eq:pol-ft}, given from Assumptions 1-4, all new policies $\pi'$ encountered during training will be in $\Pi_\text{safe}^\epsilon$ when trained on $\gtask$.
\label{thm:1}
\end{theorem}
Following from Lemma~\ref{lem:1}, we can show that actions sampled according to the projected safety policy $\bar{\pi}'$ match the constraints that $Q_\text{safe}^* < \esafe$ and therefore makes the problem equivalent to solving the CMDP $(\mathcal{S}, \mathcal{A}, P_\text{target}, \mu, r_\text{target}, \qsafe)$, where the constraints are learned in the pre-training phase. A more detailed proof is included in the supplemental material.

\section{Practical Implementation of SQRL}

To train a safety-critic and policy with SQRL on high-dimensional control problems, we need to make several approximations in our implementation of the algorithm which in practice do not harm overall performance. Our method is implemented as a layer above existing off-policy algorithms which collect and store data in a replay buffer that is used throughout training for performing updates. For our method, we store trajectories in an offline replay buffer, $\gD_\text{offline},$ when collecting samples for training the actor and critic, and additionally keep a smaller, ``on-policy'' replay buffer $\gD_\text{safe}$, which is used to store $k$ trajectories samples from the latest policy to train the safety-critic, after training the policy offline for $n_\text{off}$ steps. During pre-training, we use rejection sampling to find actions with a failure probability that are only just below the threshold $\esafe$, scored using the current safety-critic.
This equates to sampling safe but `risky' actions that encourage the agent to explore the safety boundary to correctly identify between safe and unsafe actions for a given state. Alternatively, methods like cross entropy method which perform importance-based sampling weighted by the safety-critic failure probability could also be used.

Sampling actions in this way affects the algorithm by ensuring the data being collected by our method during policy evaluation gives the safety-critic the most information about how well it determines safe and unsafe actions close to the threshold of $\esafe$. Since the behavior of the final safety policy is largely determined by the safety hyperparameters $\esafe$ and $\gsafe$, which are task and environment-specific, one additional step during pre-training is tuning both of these parameters to reach optimal pre-training performance, since it is unsafe to tune them while attempting to solve the target task.

In the fine-tuning phase, the safe sampling strategy translates into selecting safe actions according to $\qsafe^{\bar{\pi}}$ by, once again, sampling $k$ actions, masking out those that are unsafe, and then importance sampling the remaining options with probability proportional to their log probability under the original distribution output by $\pi(\cdot |s)$. In the event that no safe action is found, the safest action with lowest probability of failing is chosen. Qualitatively in our experiments, these are typically actions that stabilize the robot to avoid failure, and restrict the ability to accumulate reward for the remainder of the episode.
In our experiments in Section~\ref{sec:experiments}, we show ablation results that confirm the benefit of the exploration induced by MaxEnt RL to improve the quality of safety exploration during pre-training, yielding a better quality estimate of failure probabilities when the goal task is attempted.

%% file: 6-experiments.tex
In our experimental evaluation, we aim to answer the following questions {to evaluate how well our method transfers in this task-agnostic setting}: (1) Does our approach enable substantially safer transfer learning compared to RL without a safety-critic? (2) How does our approach compare to prior safe RL methods? (3) How does SQRL affect the performance, stability, and speed of learning? (4) Does tuning the safety threshold $\esafe$ allow the user to trade-off risk and performance?

\textbf{Environments.} 
\begin{figure*}
  \begin{minipage}{.565\textwidth}
  \centering
  \begin{subfigure}[b]{.99\linewidth}
  \centering
      \includegraphics[width=0.99\linewidth]{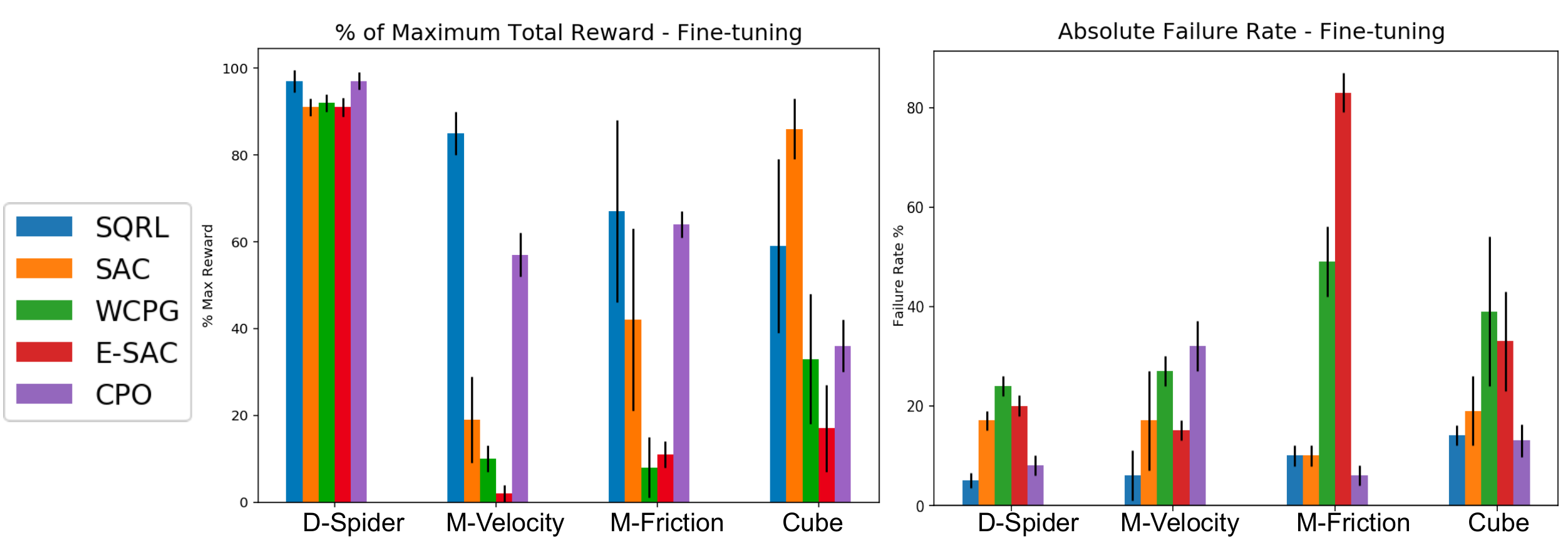}
      \label{fig:bar-plot}
      \vspace{-0.5cm}
  \end{subfigure}
  \caption{\small Final task performance and cumulative failure rate during fine-tuning. SQRL achieves good performance while also being significantly safer during learning. }
  \end{minipage}
  \hfill
  \begin{minipage}{.42\textwidth}
  \begin{subfigure}[b]{.99\linewidth}
    \includegraphics[width=1.0\linewidth]{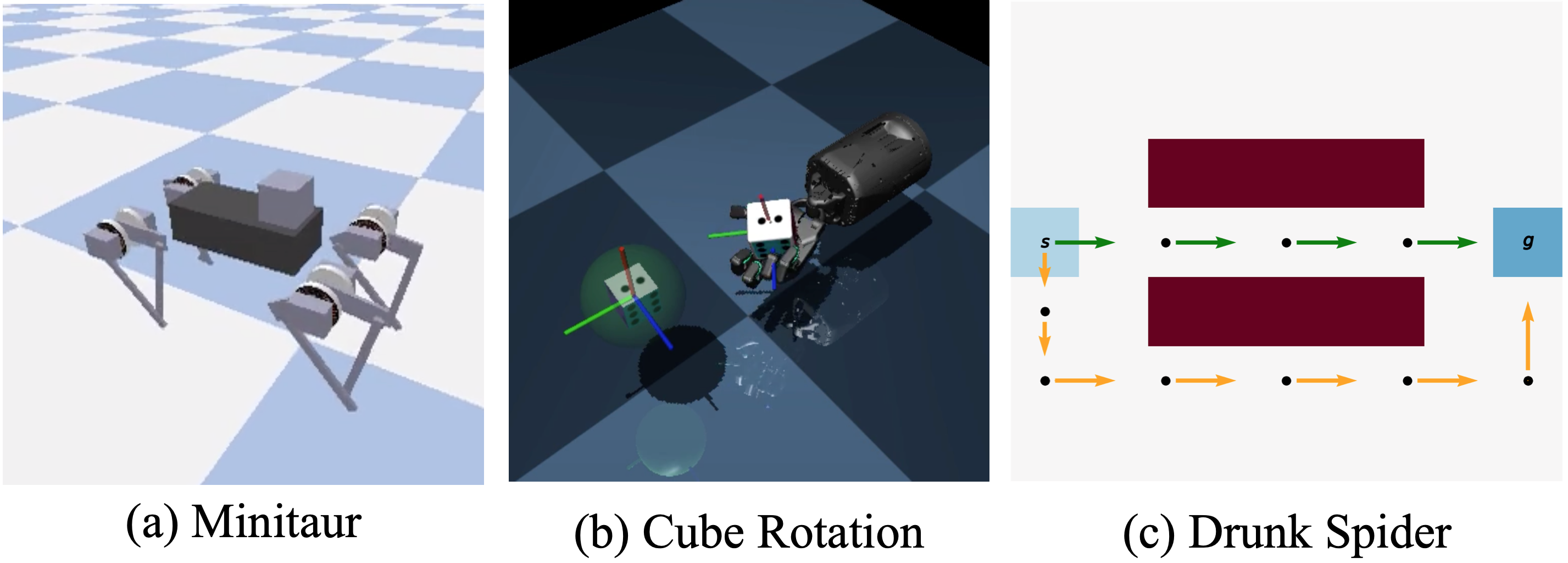}
    \vspace{-0.5cm}
  \end{subfigure}
  \caption{\small Environments used in the paper}
    \label{fig:envs}
  \end{minipage} 
\end{figure*}
To answer these questions, we design our experiments in three distinct environments, shown in Figure~\ref{fig:envs}. Inspired by~\citet{kappen2005path}, the \textbf{drunk spider} domain is a 2D navigation environment, where the agent must learn to navigate to a fixed goal position, with stochastic dynamics corresponding to action noise. There is a narrow bridge that goes directly from start to goal, with two lava pits on either side. Falling into either pit corresponds to a failure and episode termination. If the agent is unconfident about crossing the bridge without falling, it can instead choose to go around the bridge and lava pits entirely, which takes longer (\textit{i.e.}, costs more) but avoids the risk of falling. Our second environment is the \textbf{minitaur} quadruped locomotion environment from~\citet{tan2018sim}. Here, we consider the task of running at faster speeds, i.e. 0.4 m/s, after pre-training on a safer desired velocity of 0.3 m/s, with the goal of learning to run without falling. We also include a variant of this environment that increases the foot friction of the simulator during fine-tuning, showing that our method works even when there are changes in the environment dynamics.
Our last environment corresponds to \textbf{cube rotation} with a five-fingered 24 DoF ShadowHand using the environment introduced by~\citet{nagabandi2019deep}. The goal is to rotate a cube from a random initial orientation to a desired goal orientation, where the fine-tuning task requires a larger orientation change than the pre-training phase. An unsafe state corresponds to the block falling off the hand.
See Appendix~\ref{appendix:env} for details.

\subsection{The Performance and Safety of Learning}

To answer question (1), we perform an ablation of our approach where no safety-critic training is performed. This comparison directly corresponds to training standard SAC on the pre-training task, and fine-tuning the pre-trained policy on the target task, {shown in Figure~\ref{fig:safety_stability}}.
This directly compares the role the safety-critic plays in both safety and performance. To address question (2), we compare SQRL to four prior approaches for safe reinforcement learning: worst cases policy gradients (WCPG)~\cite{tang2019worst}, which learns a distribution over Q-values and optimizes predicted worst case performance, ensemble SAC (ESAC), which uses an ensemble of critics to produce a conservative, risk-averse estimate of the Q-value during learning~\cite{eysenbach2017leave}, constrained policy optimization (CPO), which adds a constraint-based penalty to the policy objective~\cite{achiam2017constrained}, and Risk-Sensitive Q-Learning (RS), which adds the risk estimate directly to the reward objective~\cite{geibel2005risk}. To compare to~\cite{geibel2005risk}, we re-use the pre-trained SQRL policy and penalize the critic loss with the safety-critic risk estimate scaled by a Lagrange multiplier. For all prior methods, we perform the same pre-training and fine-tuning procedure as our method to provide a fair comparison.

We compare all of the methods in fine-tuning performance in Figure~\ref{fig:safety_stability}. Overall, we observe that SQRL is substantially safer, with fewer safety constraint violations, than both SAC and WCPG, while transferring with better final performance in the target task. In the drunk spider environment, all approaches can acquire a high performing policy; however, only SQRL enables safety during the learning process, with only a 1\% fall rate.
In the Minitaur environment, SQRL enables the robot to walk at a faster speed, while falling around only 5\% of the time during learning, while the unconstrained SAC policy fell about 3x as often. Finally, on the challenging cube rotation task, SQRL achieves good performance while dropping the cube the least of all approaches, although SAC performs better with more safety incidents.



\begin{figure*}
\centering
    \includegraphics[width=0.99\linewidth]{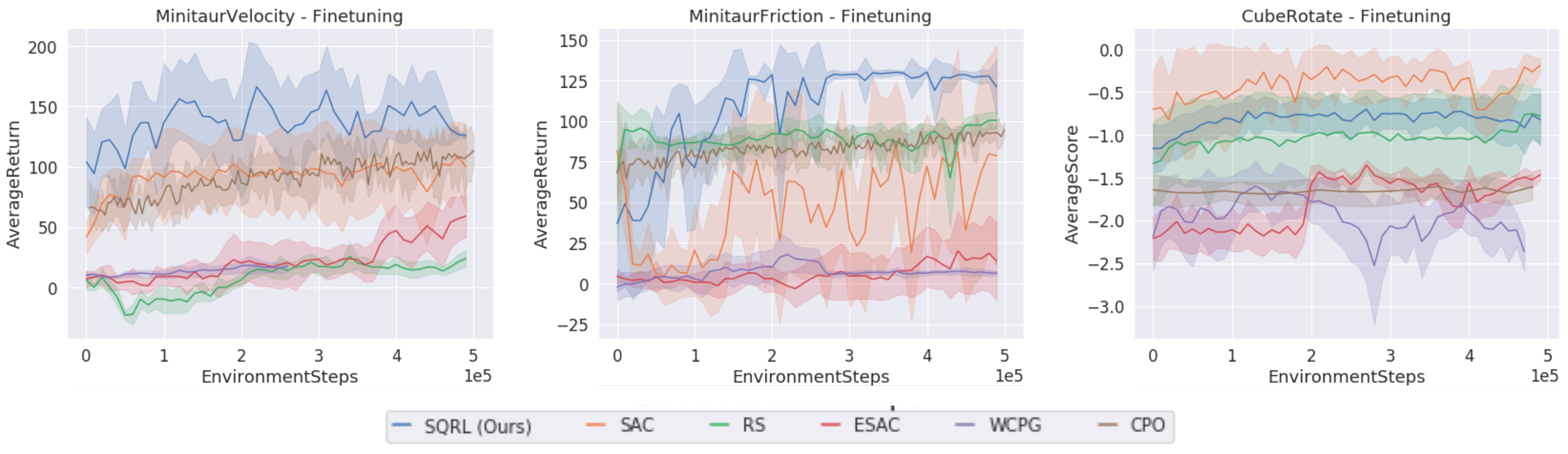}
    \vspace{-0.1cm}
    \caption{Fine-tuning curves for our method on the Minitaur and Cube target tasks, suggesting that, beyond increasing safety, SQRL's safety-critic leads to more stable and efficient learning.}
    \label{fig:safety_stability}
    \vspace{-1em}
\end{figure*}

\subsection{The Effect of Safety on Learning}

\begin{wrapfigure}{r}{0.5\textwidth}
    \vspace{-2em}
    \centering
    \includegraphics[width=1\linewidth]{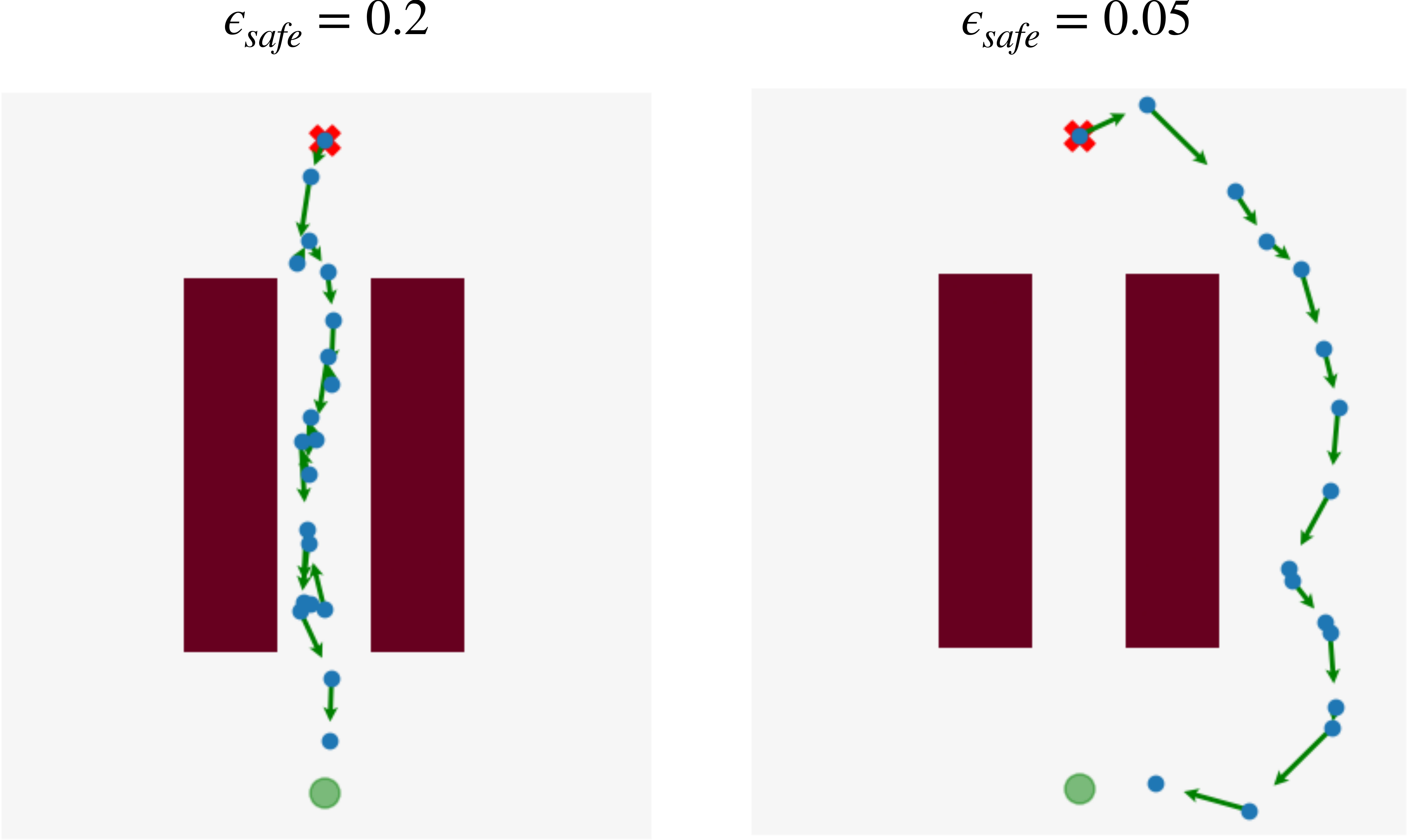}
    \vspace{-0.2cm}
    \caption{\small Qualitative results showing different trajectories corresponding to different target safety thresholds.}
    \vspace{-0.2cm}
    \label{fig:drunkspider}
\end{wrapfigure}

To answer question (3), we now aim to assess the stability and efficiency of learning when using the safety-critic. To do so, we present learning curves for SQRL and our comparisons in Figures~\ref{fig:safety_stability} and~\ref{fig:cumulative_failures} (in Appendix~\ref{appendix:training-plots}). In these curves, we find that SQRL not only reduces the chances of safety incidents (Figure~\ref{fig:cumulative_failures}), but also significantly speeds up and stabilizes the learning process (Figure~\ref{fig:safety_stability}). In particular, we observe large performance drops during SAC training, primarily due to exploration of new policies, some of which are unsafe and lead to early termination and low reward. Our approach instead mitigates these issues by exploring largely in the space of safe behaviors, producing more stable and efficient learning. We also show that when fine-tuning for the cube rotation task, the safety policy can be overly cautious (see Figures~\ref{fig:safety_stability} and~\ref{fig:cumulative_failures} in the appendix), resulting in fewer failures but lower performance than the unconstrained SAC policy. These results indicate that even when the fine-tuning task is a more difficult version of the original training task, a task-agnostic safety Q-function can guide exploration to be safer while learning to solve the new task.

\subsection{Trading Off Risk and Reward with $\esafe$}

Finally, we aim to answer question (4). We aim to evaluate whether the safety threshold, $\esafe$, can be used to control the amount of risk the agent is willing to take. To do so, we consider the drunk spider environment, evaluating two different values of our safety threshold, $\esafe=0.2$ and $\esafe=0.05$, with action noise included in the environment with magnitude $\varepsilon_\text{action}=0.2$. In Figure~\ref{fig:drunkspider}, we observe that, when using the conservative safety threshold of $\esafe=0.05$, the agent chooses to cautiously move around the two pits of lava, taking the long but very safe path to the goal location in green. Further, when we increase the safety threshold to a maximum failure probability of $\esafe=0.2$, we find that the agent is willing to take a risk by navigating directly through the two pits of lava, despite the action noise persistent in the environment. We observe that these two behaviors are consistent across multiple trials. Overall, this experiment suggests that tuning the safety threshold $\esafe$ does provide a means to control the amount of risk the agent is willing to take, which in turn corresponds to the value of the task reward and efficiency relative to the (negative) value of failure. This knob is important for practical applications where failures have varying costs.

%% file: 7-conclusion.tex
We introduced an approach for learning to be safe and then using the learned safety precautions to act safely while learning new tasks. Our approach trains a safety-critic to estimate the probability of failure in the future from a given state and action, which we can learn with dynamic programming. We impose safety by constraining the actions of future policies to limit the probability of failures. Our approach naturally yields guarantees of safety under standard assumptions, and can be combined in practice with expressive deep reinforcement learning methods. On a series of challenging control problems, our approach encounters substantially fewer failures during the learning process than prior methods, and even produces superior task performance on one task.

Our experimental results make strides towards increasing the safety of learning, compared to multiple prior methods.
However, there are a number of potential avenues for future work, such as optimizing the pre-training or curriculum generation procedures.
Relaxing some of the assumptions to guarantee that a constrained policy is safe may yield alternative ways of calculating risk and sampling safe actions under the Q-learning framework we propose.
Motivated by our analysis, another area for study is handling out-of-distribution queries to the safety-critic. 
Finally, encouraged by the strong results in simulation, we aim to apply SQRL to real-world manipulation tasks in future work.

%% file: appendix.tex

\newpage
\onecolumn
\appendix
\begin{LARGE}
\begin{center}
\textbf{Learning to be Safe: Deep RL with a Safety-Critic Supplementary Material}
\end{center}
\end{LARGE}
\maketitle
\setcounter{theorem}{0}
\setcounter{assumption}{0}

\section{Guaranteeing Safety with the Safety Critic}
\label{appendix:A}
In this section we will prove that constraining \emph{any} policy using a safety Q-function guarantees safety. Theorem~\ref{thm:1} from the main text is a special case of this result. We first clarify our notation and assumptions before stating and proving our result. Our proof will consist of the safety Q-function $\qsafe$ for a safety indicator function $\gI$, along with a safety threshold $\epsilon > 0$. 

For an arbitrary policy $\pi$, define $\bar{\pi}_{\gI,\epsilon}$ (sometimes written as $\bar{\pi}$) as the policy obtained by ``masking'' $\pi$ with $\qsafe^{\bar{\pi}}$:
\begin{equation}
    \bar{\pi}_{\gI, \epsilon}(a \mid s) \triangleq \frac{1}{Z(s)} \mathbbm{1}(\qsafe(s, a) < \epsilon) \cdot \pi(a \mid s),
    \label{eq:masking}
\end{equation}
where the normalizing constant $Z(s)$ is defined as
\begin{equation}
    Z(s) \triangleq \int \mathbbm{1}(\qsafe^{\bar{\pi}}(s, a) < \epsilon) \cdot \pi(a \mid s) da.
\end{equation}
We will denote the initial, pre-trained policy as $\pi^*$ and its safety-constrained variant as $\bar{\pi}_{\gI,\epsilon}^*$. Additionally, we define the safe set as $\gS_\text{safe} \triangleq \{s : \gI(s) = 0 \}$, and the unsafe set as $\gS_\text{unsafe} \triangleq \{s : \gI(s) = 1 \}$.

We make the following simplifying assumptions:
\begin{assumption}
The safety-critic $\hat{Q}_\text{safe}^{\bar{\pi}}$ is optimal such that, after pre-training, it can estimate the true expected future failure probability given by ${Q}_\text{safe}^{\bar{\pi}}$ in Equation~\ref{eq:failure-prob}, for any experienced state-action pair.\label{assump1-appendix}
\end{assumption}

\begin{assumption}
The transition dynamics leading into failure states have a transition probability at least $\epsilon$. That is, for all unsafe states $s' \in \gS_\text{unsafe}$, $P(s'|s, a) > \epsilon$ or $P(s'|s, a)=0$  $\forall (s, a)$.
\label{assump2}
\end{assumption}

\begin{assumption}
At every safe state $s \in \gS_\text{safe}$, there exists an action $a$ such that there is a non-zero probability of transitioning to another safe state $s' \in \gS_\text{safe}$, i.e.\ $\Pr(s'|s) > 0$. This means that at any safe state that could transition to a failure state $s'' \in \gS_\text{unsafe}$.
\label{assump3}
\end{assumption}

\begin{assumption}
The support of the pre-training data for $\qsafe$ covers all states and actions observed during fine-tuning, and the policy distribution $\pi^*$ has full support over actions.
\label{assump4}
\end{assumption}

With these assumptions in place, we formally state our general result:

\begin{theorem} \label{thm:general}
Optimizing the policy learning objective in Eq.~\ref{eq:pol-ft}, with assumptions that the safety-critic is an optimal learner and the initial policy is $\epsilon$-safe under $\ptask$, all new policies $\pi'$ encountered during training will be in $\Pi_\text{safe}^\epsilon$ when trained on $\gtask$.
\begin{equation}
    \E_{\rho_{\bar{\pi}_{\gI, \epsilon}}}\left[\sum_{t=0}^\infty \gamma^t \gI(s_t) \right] < \epsilon.
    \label{eq:thm-general}
\end{equation}
\end{theorem}
\begin{proof}


To start, recall that the safety critic $\qsafe^{\bar{\pi}^*}(s_t,a_t)$, the (optimal) Q-function for the safety indicator $\gI(s_t)$, estimates the future discounted probability of failure under the optimal safety-constrained pre-trained policy: 
\begin{equation}
    \qsafe^{\bar{\pi}^*}(s_t,a_t) \triangleq \E_{\bar{\pi}_{\gI,\epsilon}^*} \left[ \sum_{t'\geq t} \gamma^{t'-t}\gI(s_{t'})\right] \qquad 
    \label{eq:qsafe-def}
\end{equation}
At a high level, we will show, by contradiction, that if we are masking actions for $\bar{\pi}_{\gI, \epsilon}$ using this safety-critic, it cannot have a failure probability $\ge \eps$, since this would imply there is a state, action pair $(s_t, a_t)$, where $\qsafe^{\bar{\pi}^*}(s_t,a_t) \ge \epsilon$, and $\bar{\pi}_{\gI, \epsilon}(a_t | s_t) > 0$, breaking our construction of a $\qsafe$-masked policy in Eq.~\ref{eq:masking}. 

First, we expand the equation for the failure probability for one time-step (letting $\rho_{\bar{\pi}}$ be the induced state distribution by following $\bar{\pi}_{\gI, \epsilon}$) to show that if the failure probability is larger than $\eps$, this yields the contradiction.
\begin{align}
    \E_{s_t, \bar{\pi}_{\gI,\epsilon}} \left[\sum_{t'=t}^\infty \gamma^t \gI(s_t')\right] 
    &= \mathop{\E}_{s_t, \bar{\pi}_{\gI, \epsilon}} \left[
        \sum_{a_t \in \gA}
            \bar{\pi}_{\gI, \epsilon}(a_t|s_t) \underbrace{\E_{s_t' \sim P(s_t, a_t)}
                [\gamma \qsafe^{\bar{\pi}_{\gI, \epsilon}}(s_{t'}, a_{t'})]}_{(a)}
    \right] \stackrel{?}{\ge} \epsilon
    \label{eq:expanded-fp}
\end{align}
For the expected failure probability to exceed $\epsilon$ at some point $t'$, there must exist an action $a_{t'}$ in the support of $\bar{\pi}_{\gI,\epsilon}(\cdot|s_{t'})$ where the next state $s_{t'+1}$ is unsafe. This leads to the inner expectation (Eq.~\ref{eq:expanded-fp}~a) to yield a failure probability $> \epsilon$. This comes naturally from Assumption~\ref{assump2}, which also states the probability of such a transition is $> \epsilon$. In order for the cumulative discounted future probability $\sum_{t'=t+1}^\infty \gamma^{t'-t} \gI(s_{t'}) > 0$, a failure state would therefore need to be reached after some step $t'$ in the future. From Assumption~\ref{assump2}, expanding the expected future failure probability under $\bar{\pi}_\epsilon$ from $s_{t'}$ yields
%
%
\begin{equation*}
\smashoperator[l]{\mathop{\E}_{s_{t'+1} \sim P(s_{t'},a_{t'})}} \gamma \gI(s_{t'+1}) 
= \gamma P(s_{t'+1}|s_{t'},a_{t'}) \stackrel{?}{\ge} \epsilon.  
\end{equation*}
Next, let $\gA_\text{unsafe}(s_{t'}) := \{a : \bar{\pi}(a|s_{t'}) > 0, P(s_{t'+1}|s_{t'},a) > \epsilon, \gI(s_{t'+1}) = 1\}$.
Since actions that lead to unsafe states should have already been masked out by $\qsafe$ (from Equations~\ref{eq:masking},\,\ref{eq:qsafe-def}), if $\gA_\text{unsafe} \neq \emptyset$, then
\begin{equation*}
\implies 
\qsafe(s_{t'}, a_{t'}) = \mathop{\E}_{a_{t'} \sim \bar{\pi}^*(s_{t'})} \gamma \gI(s_{t'+1}) > \epsilon.
\end{equation*}
this yields a contradiction, since by the definition of a masked policy (Eq.~\ref{eq:masking}), $\bar{\pi}_{\gI, \epsilon}(a | s) > 0 \implies \qsafe(s, a) < \epsilon$.

Furthermore, under Assumption~\ref{assump3}, we are guaranteed that a safe action always exists at every state in $\gS_\text{safe}$, which means that the safety-constrained policy can always find a safe alternate action $a$ where $\Pr(s | s_{t'}, a) > 0$, where $s, s_{t'} \in \gS_\text{safe}$. 

Therefore, under no conditions will the discounted future failure probability of $\bar{\pi}_{\gI, \epsilon}$ being $\ge \epsilon$, assuming the policy does not start at an unsafe state $s_0$, i.e.\ $\mu(s) = 0 \ \forall \ s \in \{\gS_\text{unsafe}\}$.

\end{proof}
This shows that following a policy with actions masked by the safety-critic limits the probability that the policy fails, and is equivalent even if the policy was trained separately from the safety-critic, since it will be at least as safe as the safety-critic's actor. Lastly, our method of masking out unsafe actions offers a general way of incorporating safety constraints into any policy using the safety-critic (Fig.~\ref{fig:safety_mask} illustrates this concept), and works as long as there are safe actions (i.e. \(\gA_\text{safe}(s) = \{ a~:~\pi(a|s) > 0,~\qsafe(s,a) < \epsilon \} \neq \emptyset\)) in the support of that policy's action distribution.

\begin{figure}[b]
    \centering
    \includegraphics[width=0.75\linewidth]{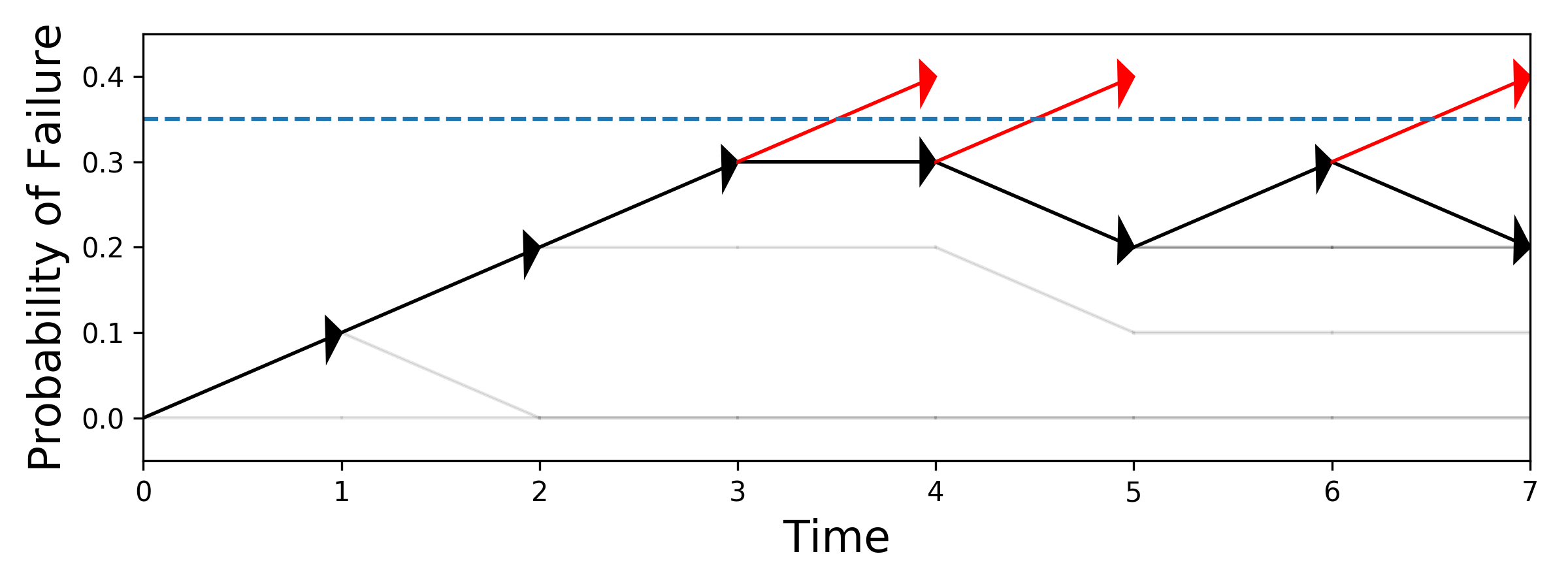}
    \vspace{-0.3cm}
    \caption{\textbf{Failure probability throughout an episode.} Arrows denote actions: black actions are those taken by our policy, gray actions are those that would have been taken by the safety critic, and red actions are those that are disallowed by the safety critic. By avoiding unsafe actions (red), the safety critic ensures that the failure probability remains below a given threshold (dashed line).}
    \label{fig:safety_mask}
    \vspace{0.2cm}
\end{figure}


A corollary of Theorem~\ref{thm:general} is that all iterates of the SQRL algorithm will be safe:
\begin{corollary}
Let $\pi_1^{SQRL}, \cdots, \pi_T^{SQRL}$ be the policies obtained by SQRL when using cost function $\gI$, safety threshold $\epsilon$, during the fine tuning stage. All policy iterates have an expected safety cost of at most $\epsilon$:
\begin{equation}
        \E_{\pi_i^{SQRL}}\left[\sum_t \gamma^t \gI(s_t) \right] \le \epsilon \qquad \forall 1 \le i \le T.
\end{equation}
\end{corollary}
\begin{proof}
At each policy update step, the policies obtained by SQRL are returned from performing one-step updates on the previous step's policy using the SQRL fine-tuning policy objective in Eq.~\ref{eq:pol-ft}. By construction, each policy iterate is obtained by masking via the current safety critic $\qsafe^{\bar{\pi}_i}$, which is computed by evaluating that :

\begin{equation}
    \pi_i^{SQRL} \triangleq \bar{\pi_i}_{\gI, \epsilon}.
\end{equation}
Applying Theorem~\ref{thm:general}, we obtain the desired result.
\end{proof}

\section{Hyperparameters}

In Table~\ref{hyper-table}, we list the hyperparameters we used across environments, which were chosen via cross-validation.

\begin{table}[h!]
\centering
\parbox{.95\linewidth}{
\caption{Default Hyperparameters used for training.}
\label{hyper-table}
\vskip 0.15in
\begin{center}
\begin{small}
\begin{sc}
\begin{tabular}{lc}
    \toprule
    \multicolumn{2}{c}{\textbf{General Parameters}} \\
    \midrule
    Learning rate & $3 \times 10^{-4}$ \\
    Number of pre-training steps & $5 \times 10^{5}$ \\
    Number of fine-tuning steps & $5 \times 10^{5}$ \\
    Layer size & 256 \\
    Number of hidden layers & 2 \\
    Hidden layer activation & ReLU \\
    Output activation & tan-h \\
    Optimizer & Adam \\
    \midrule
    \multicolumn{2}{c}{\textbf{SQRL Parameters}} \\
    \midrule
    Target Safety ($\esafe$) & $0.1$ \\
    Safety Discount ($\gamma_\text{safe}$) & $0.7$ \\
    \midrule
    \multicolumn{2}{c}{\textbf{E-SAC Parameters}} \\
    \midrule
    Number of Critics & $10$ \\
    \bottomrule
    \end{tabular}
\end{sc}
\end{small}
\end{center}
\vskip -0.1in
}
\\
\vspace{1em}
\parbox{.95\linewidth}{
\caption{Environment-Specific Hyperparameters used for training.}
\vskip 0.15in
\begin{center}
\begin{small}
\begin{sc}
\begin{tabular}{lc}
\toprule
\multicolumn{2}{c}{\textbf{DrunkSpider Parameters}} \\
\midrule
Episode length & $30$ \\
Action scale & $1$ \\
Action noise & $0.1$ \\
Number of pre-training steps & $5 \times 10^4$ \\
Safety Discount ($\gamma_\text{safe}$) & $0.65$ \\
\midrule
\multicolumn{2}{c}{\textbf{Minitaur Parameters}} \\
\midrule
Pre-training Goal Velocity & $0.3$ m/s \\
Fine-tuning Goal Velocity & $0.4$ m/s \\
Pre-training Goal Friction & $1$ \\
Fine-tuning Goal Friction & $1.25$ \\
Episode length & 500 \\
Number of pre-training steps & $5 \times 10^5$ \\
\midrule
\multicolumn{2}{c}{\textbf{CubeRotate Parameters}} \\
\midrule
Fine-tuning Goal Rotation & ($1.5 \times$ quarter rotation) \\
Episode length & 100 \\
Number of pre-training steps & $1 \times 10^6$ \\
\bottomrule
\end{tabular}
\end{sc}
\end{small}
\end{center}
\vskip -0.1in
}
\end{table}

\section{Environment Details}
\label{appendix:env}
The \textbf{MinitaurFriction} and \textbf{MinitaurVelocity} tasks are both identical in the pre-training phase, which involves reaching a desired target velocity of $0.3$ m/s at normal foot-friction. During fine-tuning, in the MinitaurFriction environment, the foot-friction is increased by $25\%$, while in the MinitaurVelocity environment, the goal velocity is increased by $33\%$ (to $0.4$ m/s). The reward function used for this environment at each timestep is $r_\text{minitaur}(s_t,a_t) = vel_g + |vel_g - vel_{cur}| - 0.01 \cdot acc$, where $acc$ is the approximate joint acceleration, calculated as $acc = a_t - 2\cdot a_{t-1} + a_{t-2}$.

In the \textbf{CubeRotate} environment, the pre-training task is to reach one of $4$ different rotated cube positions (up, down, left, and right) from a set of fixed starting positions, each of which correspond to a single quarter turn (i.e. $90$ degrees) from the starting position. We use an identical formulation of the reward to what is used in \cite{nagabandi2019deep}. For the fine-tuning task, we offset the cube by an additional eighth turn in each of those directions, and evaluate how well it can get to the goal position.

\section{Algorithm Implementation}

While our safety-critic in theory is trained estimating the Bellman target with respect to the safety-constrained policy, which is how we sample policy rollouts during safety-critic training, when computing the Bellman target, we sample future actions from the unconstrained policy. This is due to the fact that at the beginning of training, the safety-critic is too pessimistic, assuming most actions will fail, and as a result, falsely reject most safe actions that would lead to increasing reward.

Another implementation detail is the use of a smaller ``online'' replay buffer, as mentioned in Algorithm~\ref{alg:sqrl-fine}. This was used in order to get sampling to train the safety-critic to be more on-policy, since in practice, this seems to estimate the failure probabilities better, which improves the performance of the safety-constrained policy. 

\section{Training Plots}
\label{appendix:training-plots}

\begin{figure*}
\vspace{-1em}
    \centering
    \includegraphics[width=0.99\linewidth]{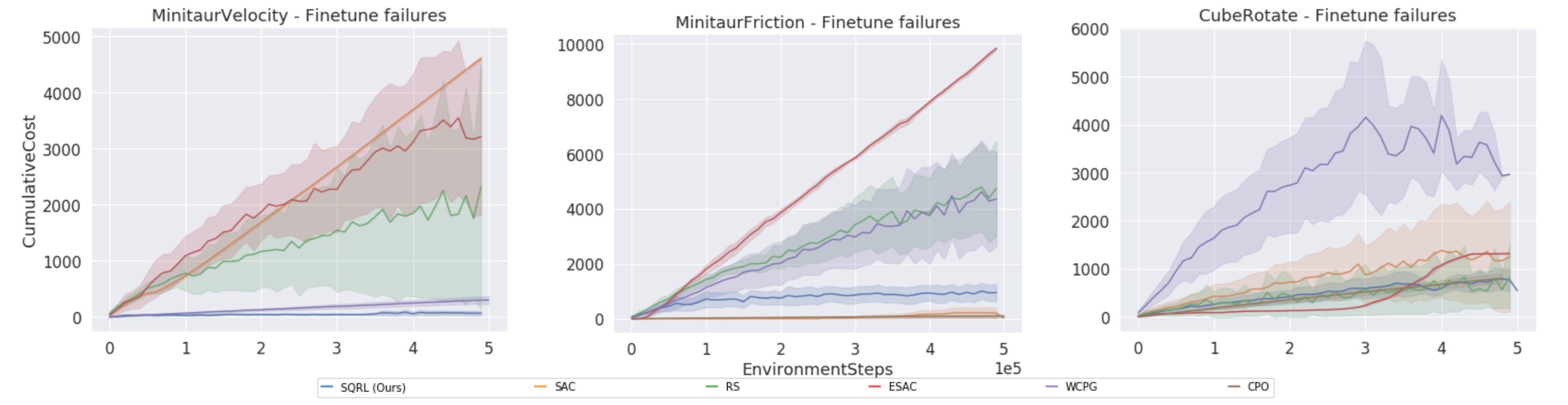}
    \vspace{-0.2cm}
    \caption{Cumulative failures over the course of fine-tuning. These results illustrate that SQRL leads to significantly fewer failures over the course of fine-tuning compared to prior safe RL approaches~\cite{tang2019worst,eysenbach2017leave}.}
    \label{fig:cumulative_failures}
\end{figure*}

In this section, we include additional training plots. We first show plots of the failures throughout fine-tuning in Figure~\ref{fig:cumulative_failures}, which illustrate how failures accumulate for different methods. We also include pre-training curves (Figs.~\ref{fig:safety_stability_pretrain_rew},~\ref{fig:safety_stability_pretrain_drop}) for the three environments. We additionally include an additional ablation plot (Fig.~\ref{fig:safety-ablation}) that shows that learning during pre-training with our method is also more stable, especially as the safety-critic is trained to be pessimistic ($\esafe = 0.05$), or optimistic ($\esafe = 0.15$), since it learns to explore safely while learning to perform the task. Each sharp drop in performance indicates that the agent experienced one or more early-reset conditions (due to unsafe-behavior) during an episode of training.

We also include an ablation (Fig.~\ref{fig:offline-ablation}) show how our method performs on the Minitaur environment during pretraining and finetuning (with a faster goal velocity) with and without online training of the safety critic. During pretraining, the online failure rate is qualitatively more stable, falling less often on average, while during finetuning, the task reward is significantly greater, thanks to mitigating the issue of an overly pessimistic critic. Finally, for additional environment-specific experiments, we first demonstrate how a SQRL agent can learn safely with physics randomization added to the Minitaur environment~\cite{tan2018sim}, where foot friction, base and leg masses are randomized. This shows that our method can also safely transfer to environments where the dynamics have changed, suggesting future work where this approach can be applied to sim-to-real transfer.


\begin{figure*}[h!]
\centering
    \includegraphics[width=0.99\linewidth]{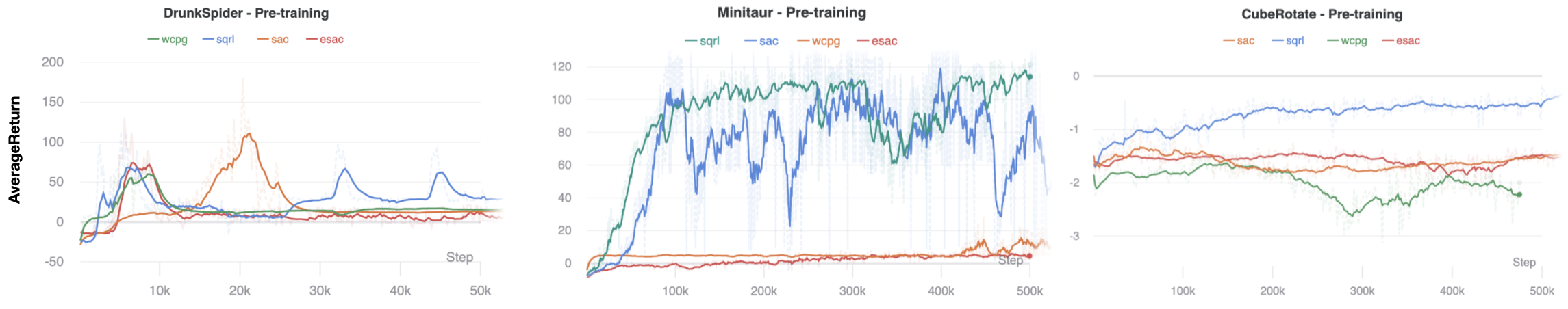}
    \vspace{-0.3cm}
    \caption{Pre-training curves for our method and comparisons for the Point-Mass, Minitaur and Cube environments. These results suggest that learning a policy alongside the safety-critic even during pre-training can lead to better performance. For the point-mass environment, the agent gets significantly more reward at the end of the trajectory if it learns to safely walk through the bridge without falling.}
    \label{fig:safety_stability_pretrain_rew}
    \vspace{-0.2cm}
\end{figure*}

\begin{figure*}[h!]
\centering
        \includegraphics[width=0.32\linewidth]{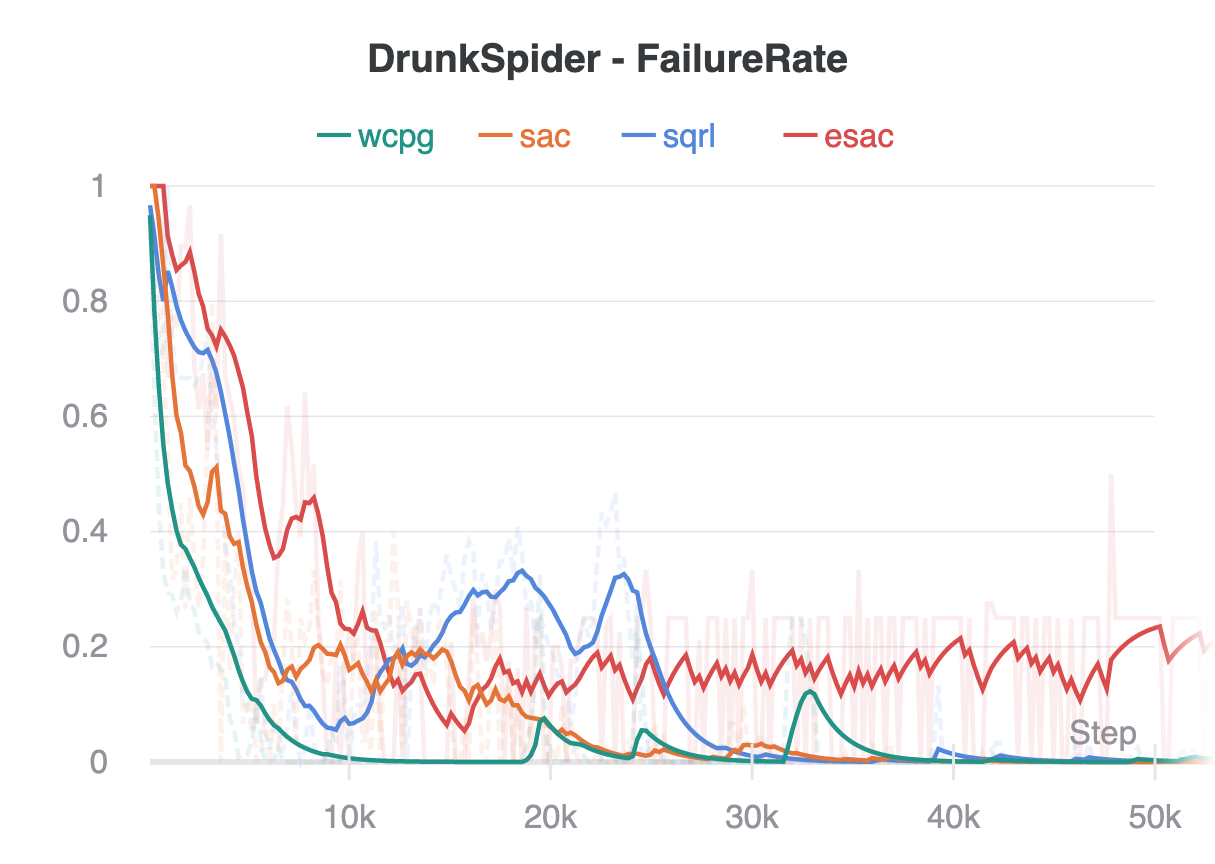}
        \includegraphics[width=0.32\linewidth]{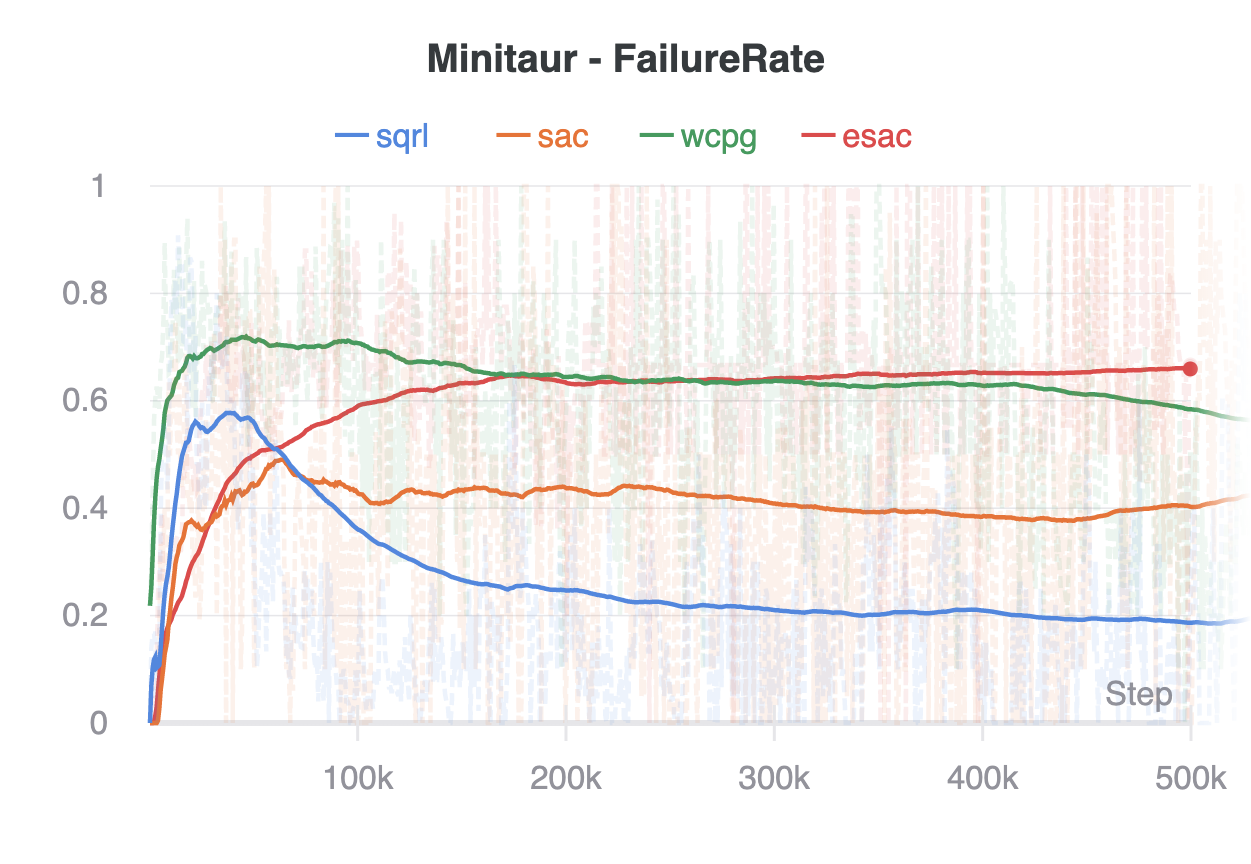}
        \includegraphics[width=0.32\linewidth]{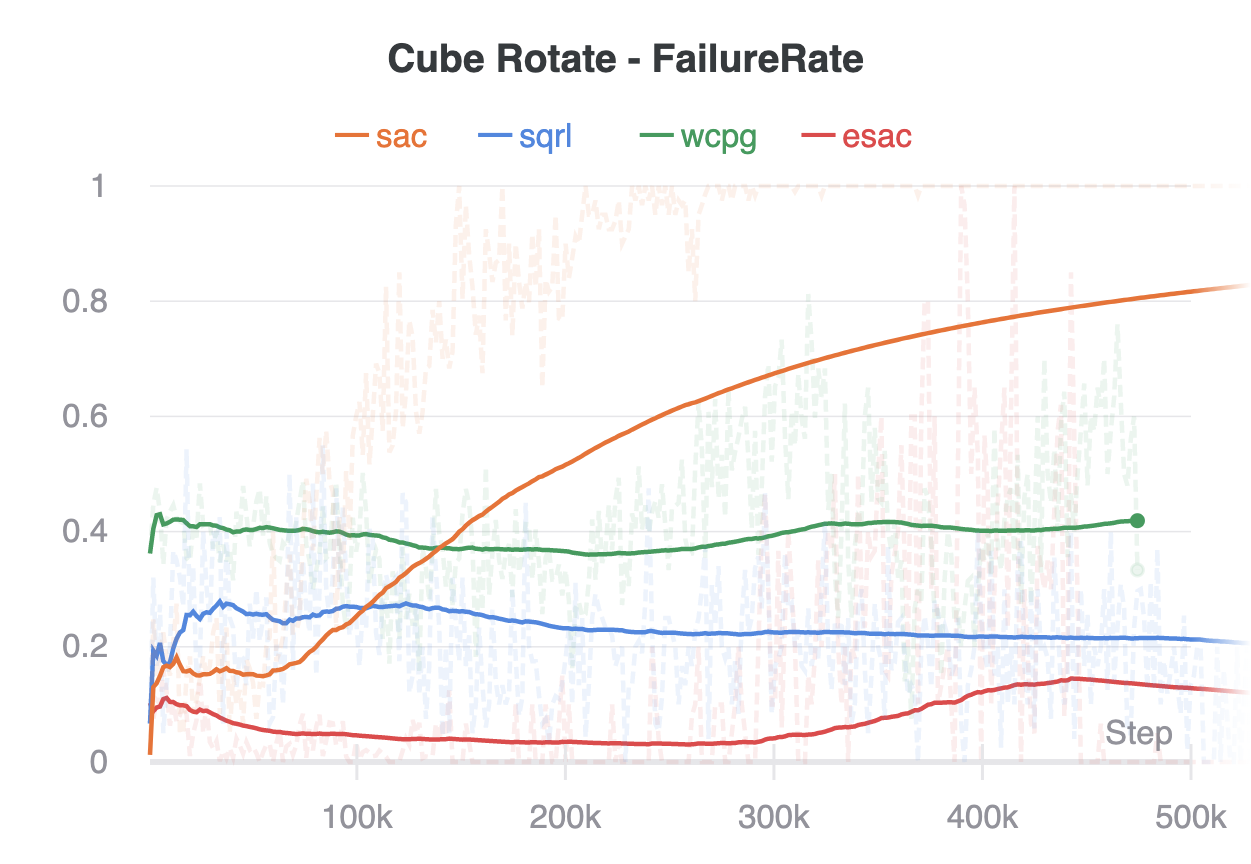}
    \vspace{-0.3cm}
    \caption{Pre-training failure rates over for our method and comparisons for the Point-Mass, Minitaur and Cube environments. These results suggest that learning a policy alongside the safety-critic shows improvements in safety even during pre-training.}
    \label{fig:safety_stability_pretrain_drop}
    \vspace{-0.3cm}
\end{figure*}

\begin{figure*}[t]
    \centering
    \begin{subfigure}[b]{.45\linewidth}
        \includegraphics[height=0.655\linewidth]{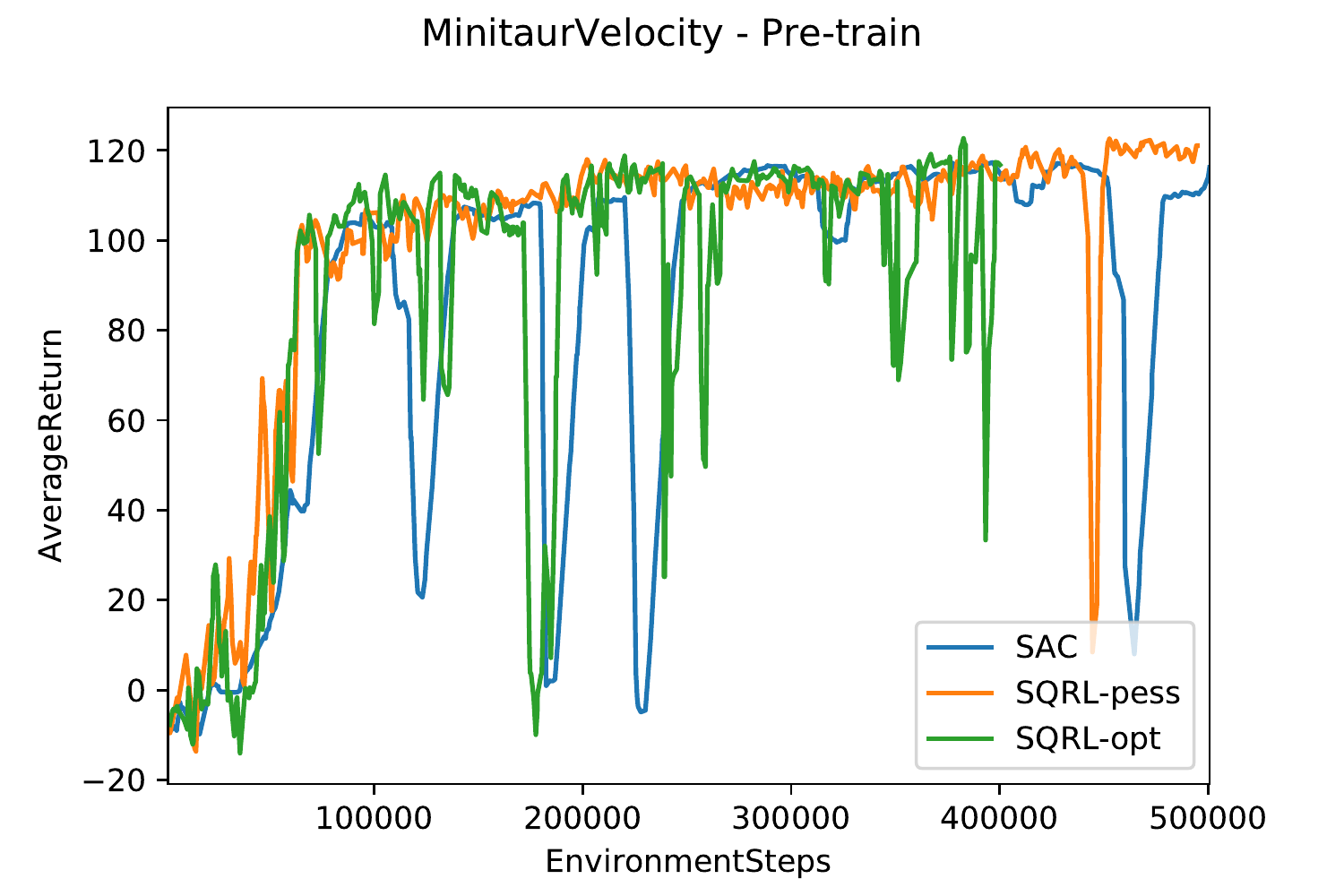}
        \caption{Using a more optimistic $(\esafe = 0.15)$ or pessimistic $(\esafe = 0.05)$ safety threshold affects the stability of training.}
        \label{fig:safety-ablation}
    \end{subfigure}
    \quad
    \begin{subfigure}[b]{.45\linewidth}
        \includegraphics[height=0.655\linewidth]{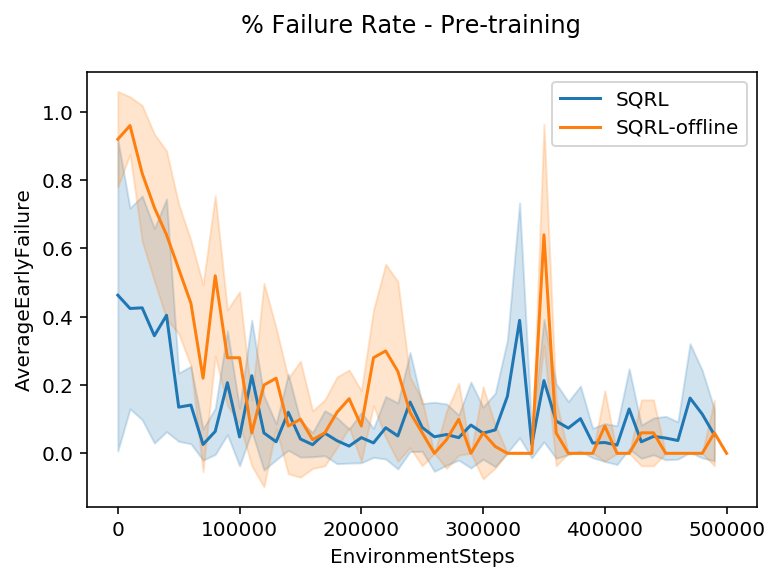}
        \caption{We visualize the failure rate of SQRL with and without online training of the safety critic during pre-training, with $\esafe=0.15$.}
    \label{fig:offline-ablation}
    \end{subfigure}
    \caption{Ablations of SQRL design choices.}
\end{figure*}
